\documentclass[twoside,11pt]{article}
\usepackage{jagi}
\usepackage{mathrsfs}
\usepackage{amsmath}
\newtheorem{question}[theorem]{Question}

\jagiheading{10}{1}{24--45}{2019}{2019-07-30}{2019-10-19}{10.2478/jagi-2019-0003}
{S.\ Alexander}

\ShortHeadings{Intelligence via ultrafilters}{S.\ Alexander}

\author{\name Samuel Allen Alexander
\email samuelallenalexander@gmail.com \\
\addr Quantitative Research Analyst\\
The U.S.\ Securities and Exchange Commission\\
New York Regional Office
}

\editor{Florin Popescu}

\title{Intelligence via ultrafilters: structural properties of some intelligence
comparators of deterministic Legg-Hutter agents}

\begin{document}

\maketitle

\begin{abstract}
    Legg and Hutter, as well as subsequent authors, considered intelligent agents
    through the lens of interaction with reward-giving
    environments, attempting to
    assign numeric intelligence measures to such agents, with the guiding
    principle that a more intelligent agent should gain higher rewards
    from environments in some aggregate sense.
    In this paper, we consider a related question: rather than measure
    numeric intelligence of one Legg-Hutter agent, how can we compare the
    relative intelligence
    of two Legg-Hutter agents? We propose an elegant answer based on
    the following insight:
    we can view Legg-Hutter agents as candidates in an election, whose voters are
    environments, letting each environment vote (via its rewards) which
    agent (if either) is more intelligent. This leads to an abstract
    family of comparators simple enough that we can prove
    some structural theorems about them. It is an open question whether
    these structural theorems apply to more practical intelligence measures.
\end{abstract}

\section{Introduction}

Who is more intelligent, $A$ or $B$? This is a paper about
how to formalize a definition of the relative intelligence of
agents who interact deterministically with deterministic
environments based on observations and
rewards\footnote{Essentially, the agents and environments
of \emph{general reinforcement learning}, but with a more universal
point of view.}
(formalized below).
Starting with the landmark paper of \citet{hutter2007}, various
attempts \citep{hernandez, hibbard, legg2013approximation, gavane} have been
made to formalize intelligence measures of such
agents\footnote{Authors differ on whether or not to require determinism.}.
We will refer to such an agent as a \emph{deterministic Legg-Hutter agent},
or \emph{DLHA} for short, to distinguish it from other agent notions.
DLHAs' intelligence can be defined in
various ways using real numbers (for example, by
means of infinite series involving Kolmogorov complexities).
The underlying principle is that a DLHA with higher intelligence should
earn higher rewards from environments, in some aggregate sense.
Thus, we are talking about a type of intelligence one would consider if
one were evaluating a general-purpose utility-maximizing AI for usage
in many well-defined, rational, emotionless tasks.

Aside from quantifying and comparing the intelligence of individual DLHAs, this field
of research has the potential to shed light on (at least certain aspects of)
intelligence itself,
abstracted over many DLHAs.
To that end, we intend to present a notion of the relative
intelligence of DLHAs, which notion exhibits
\emph{structural properties}. What we mean by
a ``structural property'' is roughly the following: if one knows something about the
intelligence of an agent $A$, a ``structural property'' is a theorem which allows one
to infer something about the intelligence of some agent $A'$ which is a variation of
$A$ in some sense.

Here is an example to make the notion of a ``structural property'' more
concrete. Suppose $A,B,A',B'$ are agents, and we have some way of letting
$A$ and $B$ (resp.~$A'$ and $B'$) ``team up'' into a new agent $C$ (resp.~$C'$).
It might seem intuitively plausible that if $A$ is more
intelligent than $A'$, and if $B$ is more intelligent than $B'$, then the team
$C$ obtained from $A$ and $B$ should be more intelligent than the team $C'$
obtained from $A'$ and $B'$. This property, ``higher-intelligence team-members
make higher-intelligence teams,'' if it held, would be an example of a structural property
of intelligence. Later in this paper, we will explore this property for
two specific notions of DLHA teamwork.

In the future, we hope that the investigation of structural properties of an
algebraic nature will help to reveal qualitative facts about intelligence.
As a precedent, one of the structural properties we prove (Proposition
\ref{trivialquitterproposition}) helped us realize that if environments can
give negative rewards, then the relative intelligence of different DLHAs
really ought to depend on how risk-averse we are; more on this below.

Rather than directly compute a real number intelligence measure for any
particular DLHA, we instead focus on the simpler problem of
how to determine whether one DLHA is more intelligent
than another\footnote{In more practical, less universal contexts, much work has
been done on comparing agents, see \cite{balduzzi} (section 1) for a collection
of references. Many practical methods for comparing agents turn out to be non-transitive,
but the abstract method we introduce is transitive (Lemma \ref{transitivityLemma}).}.

We will introduce an abstract intelligence comparison
based on the following insight:
when considering whether DLHA $A$ is more intelligent than DLHA $B$,
we can imagine that different environments are \emph{voters} voting in an election with
three candidates\footnote{We are not the first to apply election theory to unexpected
areas by reinterpreting the voters and candidates. For non-AI uses of this technique,
see \citep{okasha, stegenga}.}. Those three
candidates are: ``$A$ is more intelligent'',
``$B$ is more intelligent'', and ``$A$ and $B$ are equally intelligent''.
Arrow's impossibility theorem \citep{encyclopedia} would crush all hope of
easily\footnote{Technically speaking, the way environments vote in different elections
is constrained, with complicated relationships binding how
environments vote in different such elections. Thus, it is still possible that
other solutions to the voting problem might exist, since Arrow's impossibility theorem
assumes that
how a voter votes in one election is independent of how that voter votes
in a different election. But
because of the complicated nature of the constraints, I conjecture it
would be difficult to actually exploit this loophole.}
obtaining a
reasonable non-dictatorial solution in this manner if there were only finitely many
environments
casting votes, but that is not the case when there are infinitely many environments
casting votes.
We must instead turn to the infinite-voters version of
Arrow's impossibility theorem \citep{kirman} \citep[see also][]{fishburn},
which \emph{does} admit non-dictatorial solutions in terms of ultrafilters.

The paper is structured as follows:
\begin{itemize}
    \item Section \ref{preliminariessection} contains preliminaries:
    a formalization of the agent-environment model, and a definition of ultrafilters.
    \item
    In Section \ref{maindefinitionsection} we define a family of intelligence
    comparators---functions that compare intelligence of
    DLHAs---in terms of ultrafilters, and
    establish some basic lemmas about these comparators.
    \item
    In Section \ref{structuralpropertiessection}, we prove some structural
    properties about our intelligence comparators.
    \item
    In Section \ref{objectionsection}, we discuss some anticipated objections
    to our intelligence comparators and to the background framework.
    \item
    In Section \ref{openquestionsection}, we state some open questions about
    intelligence measures defined by other authors.
    \item
    In Section \ref{conclusionsection}, we summarize and make concluding remarks.
\end{itemize}

\section{Preliminaries}
\label{preliminariessection}

We will begin with a formal definition of deterministic Legg-Hutter agent
(DLHA) and of
deterministic environment.
We envision an agent being placed in an environment, where the agent
receives an initial numerical reward, receives an initial observation, takes an initial
action, receives a numerical reward for that action, receives an observation of how the
environment has changed, takes a second action, and so on \emph{ad infinitum}.
This high-level vision basically agrees with \citet{hutter2007}
and other authors,
except that we require both agent and environment to
be deterministic, whereas Legg and Hutter allowed both agent and environment to have
an element of randomness.
At first glance this vision might seem quite removed from real-world
agents, but the abstract numbers which we refer to as ``rewards'' can
stand in for any number of realistic quantities such as: amounts of gold
extracted from a mine; amounts of dopamine released by a brain; amounts of positive
ratings on a website; etc. Similarly, a single numerical ``observation'' can
encode any sort of real-world observation we can think of (just as any computer file
is really just a long string of binary, i.e., a single number).
This model seems flexible enough to accomodate many, if not all, practical uses
to which one would want to put a general-purpose AI.
For example, if we want the general-purpose AI to translate documents, we can concoct
an environment which incentivizes that; if we want the general-purpose AI to analyze
protein-folding, we can concoct an environment which incentivizes that; and so on.
To achieve good rewards across the universe of all environments,
such an AI would need to have (or appear to have) creativity (for those
environments intended to reward creativity),
pattern-matching skills (for those environments intended to reward pattern-matching),
ability to adapt and learn (for those environments which do not explicitly
advertise what things they are intended to reward, or whose goals change over time), etc.
See \citet{hutter2007} for further discussion of how this model relates
to machine intelligence (and for caveats about the agent-environment
model).

To simplify the mathematics\footnote{To be clear, the results in this paper
could be generalized so as to make this simplification unnecessary.
But if we went that route, all the
mathematics would be more tedious, and there would be little additional
insight gained. See \citet{hutter2007} for
justification and caveats about the decision to restrict to environments
with convergent reward sequences.}, we want to avoid the awkward situation where
the reward-sequence an agent receives adds up to $\infty$, or to $-\infty$,
or diverges (since we will allow negative rewards). There are different ways
to avoid this. For example, one
approach would be to require that the $n$th reward be bounded
between $\pm 2^{-n}$. For increased
generality\footnote{We will see in Section \ref{quittersSubSection} that the
more we restrict
the universe of environments, the more structural properties will result. We prefer
to initially lay the foundations as general as possible, and only specialize as
needed.},
we employ (in its debut appearance) a more
abstract approach: we will first define what we call a deterministic pre-environment,
with no such constraints, and then we will define a deterministic environment to be
a deterministic pre-environment whose reward sequences (for every arbitrary DLHA) converge.

\begin{definition}
\label{EnvironmentAndAgent}
    (Environment and Agent)
    \begin{enumerate}
        \item
        A \emph{deterministic pre-environment} is a function $e$ which takes as input a finite
        (possibly empty) sequence
        $a_1,\ldots,a_n$ of
        natural numbers, called \emph{actions}. It outputs a pair 
        $e(a_1,\ldots,a_n)=(r,o)$, where $r\in\mathbb R$ is called a \emph{reward}
        and $o\in\mathbb N$ is called an \emph{observation}.
        \item
        A \emph{deterministic Legg-Hutter agent} (or \emph{DLHA}) is
        a function $A$ which takes as input a finite sequence
        $r_1,o_1,\ldots,r_n,o_n$ of reward-observation pairs, and outputs an action
        $A(r_1,o_1,\ldots,r_n,o_n)$.
        \item
        Suppose $A$ is a DLHA and $e$ is a deterministic pre-environment.
        We define the \emph{reward-observation-action sequence
        determined by letting $A$ play in $e$} to be the infinite sequence
        $r_1,o_1,a_1,r_2,o_2,a_2,\ldots$ defined inductively as follows:
        \begin{itemize}
            \item
            The initial reward and initial observation,
            $r_1,o_1=e(\langle\rangle)$, are obtained by plugging the empty
            action-sequence into $e$.
            \item
            The initial action
            $a_1=A(r_1,o_1)$ is obtained by plugging $r_1,o_1$ into $A$.
            \item
            Assume $r_1,o_1,a_1,\ldots,r_i,o_i$ have been defined.
            We define $a_i$,
            the action which $A$ performs
            in response to reward-observation sequence $r_1,o_1,\ldots,r_i,o_i$,
            to be $a_i=A(r_1,o_1,\ldots,r_i,o_i)$.
            \item
            Assume $r_1,o_1,a_1,\ldots,r_i,o_i,a_i$ have been defined.
            We define $r_{i+1},o_{i+1}$, the reward and observation
            produced in response to action-sequence $a_1,\ldots,a_i$,
            to be $r_{i+1},o_{i+1}=e(a_1,\ldots,a_i)$.
        \end{itemize}
        \item
        A deterministic pre-environment $e$ is a \emph{deterministic environment} if,
        for every DLHA $A$, the \emph{total reward} $r_1+r_2+\cdots$
        which \emph{$A$ achieves
        when $A$ plays in $e$}, converges.
        \item
        If $A$ and $B$ are DLHAs and $e$ is a deterministic environment, we
        say $A$ \emph{outperforms}
        $B$ on $e$ if the total reward which $A$ achieves when $A$ plays in $e$, is larger than
        the total reward which $B$ achieves when $B$ plays in $e$.
        We define what it means for
        $A$ to \emph{underperform} $B$ on $e$, or for $A$ and $B$
        to \emph{perform equally well}
        on $e$, in similar ways.
    \end{enumerate}
\end{definition}

\begin{example}
    \begin{itemize}
        \item The function $A(r_1,o_1,\ldots,r_i,o_i)=0$ is a trivial DLHA which
        totally ignores the environment and blindly performs the same action
        over and over.
        \item The function $e(a_1,\ldots,a_i)=(1,0)$ is a trivial deterministic pre-environment
        which ignores a DLHA's actions, instead always blindly rewarding the DLHA
        with a reward of $r=1$, and never allowing any new observations. This
        deterministic pre-environment might be thought of as a sensory-deprivational paradise
        where DLHAs receive constant injections of reward from an immutable void.
        This deterministic pre-environment
        is \emph{not} a deterministic environment, because the total reward
        it grants to any DLHA diverges to $+\infty$.
        \item The function $e(a_1,\ldots,a_i)=(0,0)$ is a deterministic
        pre-environment similar to
        the previous one, except that it never grants DLHAs any reward. It is a
        deterministic environment because the total reward it grants
        any DLHA converges (to $0$).
        \item For a nontrivial example, choose a single-player video-game $V$ and 
        consider a function $e(a_1,\ldots,a_i)=(r_{i+1},o_{i+1})$
        where $o_{i+1}$ encodes the image displayed on screen
        by $V$ after the player presses buttons encoded by $a_1,\ldots,a_i$.
        Assuming $V$ displays a numerical score on the screen,
        we can take $r_{i+1}$ to be the currently displayed score minus the score which was
        displayed in $o_i$ (or minus $0$
        if $i=0$). This function $e$ is a deterministic pre-environment. There are various
        conditions on $V$ which would suffice to make $e$ an environment. For example,
        if $V$ cannot be played forever (so the score eventually always freezes);
        or if scores in $V$ never decrease, and have some limiting max score;
        etc.
    \end{itemize}
\end{example}

One might be tempted to wonder whether it is possible to find one single, particularly
clever deterministic environment such that each DLHA's intelligence can be defined
simply as the total reward it earns in that environment. The following lemma will show
that this is impossible: no single environment $e$ can, by itself, serve as a good
measure of intelligence.

\begin{lemma}
\label{inadequacylemma}
    (Inadequacy of any single deterministic environment)
    Let $e$ be any deterministic environment.
    For every DLHA $A$, there exists some DLHA $B$ such that $A$ and $B$ achieve
    the same total reward when they play in $e$, and yet $B$'s output never
    depends on which rewards or observations are input into it---loosely speaking,
    $B$ ignores everything.
\end{lemma}

\begin{proof}
    Let $r^0_1,o^0_1,a^0_1,\ldots,r^0_i,o^0_i,a^0_i,\ldots$ be the
    reward-observation-action sequence determined by letting $A$ play on $e$.
    Define $B$ by $B(r_1,o_1,\ldots,r_i,o_i)=a^0_i$, note that
    $B(r_1,o_1,\ldots,r_i,o_i)$ only depends on $i$. Informally speaking, $B$
    completely ignores everything and instead blindly regurgitates the actions
    which $A$ takes on $e$. By construction, $B$ outputs the exact same actions
    as $A$ on the particular environment $e$, so $A$ and $B$ achieve the same
    total reward when they play in $e$.
\end{proof}

The proof of Lemma \ref{inadequacylemma} recalls the following words from
\citet{sanghi2003computer}: ``...the reason that I.Q.\ test success can be
automated with comparative ease is that administering an I.Q.\ test requires
little intelligence---it requires comparatively little more than
giving a list of questions with known answers.''
Lemma \ref{inadequacylemma} shows that if intelligence were measured by performance
on any single deterministic environment, then every DLHA would be exactly as intelligent
as some DLHA that ignores the environment---a property which clearly should not hold of
a reasonable intelligence measure.

In order to arrive at an abstract method of comparing the intelligence of different
DLHAs, we will consider deterministic environments to be voters who vote in an election.
The question is how to use these votes to decide the election.
Since there are infinitely many deterministic environments, that means there are
infinitely many voters,
which scenario is thoroughly investigated by \citet{kirman}.
For the infinite voter case, Kirman and Sondermann showed that
reasonable solutions are intimately
related to the following mathematical logical device:

\begin{definition}
\label{ultrafilterdefn}
    An \emph{ultrafilter on $\mathbb N$} (hereafter simply an \emph{ultrafilter})
    is a collection $\mathscr U$ of subsets of $\mathbb N$ satisfying the following requirements:
    \begin{itemize}
        \item
            (Properness) $\emptyset\not\in\mathscr U$.
        \item
            (Upward Closure) For every $X\in \mathscr U$ and
            every $X'\subseteq \mathbb N$, if
            $X'\supseteq X$, then $X'\in\mathscr U$.
        \item
            ($\cap$-closure) For every $X,Y\in\mathscr U$, the
            intersection $X\cap Y\in\mathscr U$.
        \item
            (Maximality) For every $X\subseteq \mathbb N$,
            either $X\in\mathscr U$ or the complement
            $X^c\in\mathscr U$.
    \end{itemize}
    An ultrafilter is \emph{free} if it contains no singleton $\{n\}$ for any
    $n\in\mathbb N$.
\end{definition}

For example, for each $n\in\mathbb N$, the set-of-subsets
\[
\{X\subseteq\mathbb N\,:\, n\in X\}
\]
is an ultrafilter, but not a free ultrafilter.
It is not difficult to prove that every non-free ultrafilter has the above form.

The following theorem is well-known, and we state it here without proof.
All the proofs of this theorem are non-constructive: in a sense which can be
made formal, it is impossible to actually exhibit a concrete example of a
free ultrafilter.

\begin{theorem}
\label{ultrafilterexistence}
    There exists a free ultrafilter.
\end{theorem}

There are two competing intuitions one can use to reason about an ultrafilter,
and these intuitions almost seem to contradict each other.
\begin{itemize}
\item
The first intuition is
that ``$X\in\mathscr U$'' can be read as ``$X$ contains almost every natural number''. Through this intuitive lens, the $\cap$-closure property seems
obvious, whereas the Maximality property seems implausible.
\item
The second intuition is that ``$X\in\mathscr U$'' can be read as
``$X$ is a winning bloc of natural numbers'' (in an electoral sense).
Through this lens, the Maximality property seems obvious (when two candidates
compete in an election, one of them must win), whereas the $\cap$-closure property
seems implausible.
\end{itemize}
Theorem \ref{ultrafilterexistence} is
profound because it says these two
intuitions can be reconciled in a non-degenerate way.

\section{An abstract intelligence comparator}
\label{maindefinitionsection}

We will arrive at an abstract method of comparing DLHAs' intelligence by means of
an imaginary election. The method we arrive at will not itself directly involve elections:
elections are merely the heuristic used to obtain our definition, and do not directly
feature in our definition.
So in motivating the definition, we will intentionally
speak about elections in less than full formality. For full formality, see \citet{kirman}.

We would like to answer the question,
``who is more intelligent, $A$ or $B$?'', by letting different
deterministic environments vote. In the election, there are three
candidates: ``$A$ is more intelligent'',
``$B$ is more intelligent'', and ``$A$ and $B$ are equally intelligent''.
A deterministic environment $e$ is considered to rank these three candidates,
from most to least
preferable, as follows:
\begin{enumerate}
    \item
    If $A$ earns more reward than $B$ when run on $e$,
    then $e$ ranks ``$A$ is more intelligent'' most preferable,
    followed by ``$A$ and $B$ are equally intelligent'', followed
    by ``$B$ is more intelligent''.
    \item
    If $B$ earns more reward than $A$ when run on $e$,
    then $e$ ranks ``$B$ is more intelligent'' most preferable,
    followed by ``$A$ and $B$ are equally intelligent'', followed
    by ``$A$ is more intelligent''.
    \item
    If $A$ and $B$ earn the same reward when run on $e$,
    then $e$ ranks ``$A$ and $B$ are equally intelligent'' most preferable,
    followed by ``$A$ is more intelligent'',
    followed by ``$B$ is more intelligent''.
\end{enumerate}
The order of the last two preferences in the latter case is arbitrary.
It would be more natural to simply let $e$ vote on a single winner,
but we make $e$ rank the three candidates because that is the form
of election which
Kirman and Sondermann
considered.
It remains now to specify how to use the votes to determine a winner.

A method of turning voter preferences into a group-preference is
called a \emph{social welfare function}.
A social welfare function is a \emph{dictatorship} if there is a particular voter
(called a \emph{dictator}) whose preference
always equals the group-preference given by the social welfare function.
Using a dictatorship to settle the election in
the above paragraph would amount to defining relative intelligence entirely by performance
in one specific fixed deterministic environment,
which is undesirable (by Lemma \ref{inadequacylemma}).

Two other desirable properties of a social welfare function are:
\begin{itemize}
    \item
    (Unanimity)
    If every voter agrees to prefer candidate $A$ over candidate $B$, then
    the group-preference prefers candidate $A$ over candidate $B$.
    \item
    (Independence)
    Given any two voter-preference-sets and any
    two candidates, if the relative voter-preferences
    between the two candidates are the same in the two voter-preference-sets, then the
    corresponding group-preferences between the two candidates are the same.
\end{itemize}
Arrow's impossibility theorem states that if there are finitely many voters, then
no social welfare function can satisfy (Unanimity) and (Independence) without being
a dictatorship. See \citet{encyclopedia}
to get a better feel for what is at stake here.
Fortunately, non-dictatorial solutions are possible when there are infinitely many
voters.
The following theorem is an immediate corollary of \citet{kirman}:

\begin{theorem}
\label{shoulderofgiants}
    For any countably infinite
    sequence $\vec{e}=(e_0,e_1,\ldots)$ of deterministic environments, suppose the
    members of $\vec{e}$
    shall rank the following three candidates: ``$A$ is more intelligent'',
    ``$B$ is more intelligent'', and ``$A$ and $B$ are equally intelligent''.
    Let $\Sigma$ be the corresponding set of social welfare functions satisfying
    (Unanimity) and (Independence).
    \begin{enumerate}
        \item
        For every social welfare function $\sigma\in\Sigma$, there is exactly one
        ultrafilter $\mathscr U_\sigma$ with the following property:
        Any way $e_0,e_1,\ldots$ rank the three candidates, for every two candidates
        $x$ and $y$,
        if
        \[
            \{n\,:\,\mbox{$e_n$ prefers $x$ over $y$}\} \in \mathscr U_\sigma,
        \]
        then $\sigma$ says that $x$ is preferred
        over $y$ in the corresponding group-preference.
        \item
        The map $\sigma\mapsto \mathscr U_\sigma$ is a surjection onto the set of ultrafilters
        (in other words, every ultrafilter $\mathscr U$ is equal to $\mathscr U_\sigma$ for
        some social welfare function $\sigma\in\Sigma$).
        \item
        A social welfare function $\sigma\in\Sigma$ is a dictatorship if and only if
        $\mathscr U_\sigma$ is not free.
    \end{enumerate}
\end{theorem}

\begin{proof}
    By Theorem 1 and Proposition 2 of \citet{kirman}.
\end{proof}

In particular, since Theorem \ref{ultrafilterexistence} says there exists a free
ultrafilter, Theorem \ref{shoulderofgiants} implies there is a non-dictatorship
social welfare function satisfying (Unanimity) and (Independence).

Motivated by Theorem \ref{shoulderofgiants}, we will define a family of intelligence
comparators.

\begin{definition}
    By an \emph{electorate}, we mean a
    pair $E=(\vec{e}, \mathscr U)$ where
    $\vec{e}=(e_0,e_1,\ldots)$ is any countably infinite sequence of deterministic environments
    and $\mathscr U$ is an ultrafilter.
\end{definition}

\begin{definition}
\label{MainDefinition}
    For every electorate $E=(\vec{e},\mathscr U)$,
    we define relations $>_{E}$, $<_{E}$ and $=_{E}$ on DLHAs as follows.
    Let $A$ and $B$ be any two DLHAs.
    \begin{enumerate}
        \item
        If
        \[
            \{n\,:\,\mbox{$A$ outperforms $B$ on environment $e_n$}\}\in\mathscr U,
        \]
        we declare $A>_{E}B$
        (and say $A$ is \emph{more intelligent than $B$ according to $E$}).
        \item
        If
        \[
            \{n\,:\,\mbox{$B$ outperforms $A$ on environment $e_n$}\}\in\mathscr U,
        \]
        we declare $A<_{E}B$
        (and say $A$ is \emph{less intelligent than $B$ according to $E$}).
        \item
        If
        \[
            \{n\,:\,\mbox{$A$ earns the same reward as $B$ on $e_n$}\}\in\mathscr U,
        \]
        we declare $A=_{E}B$
        (and say $A$ and $B$ are \emph{equally intelligent according to $E$}).
    \end{enumerate}
\end{definition}

If $\mathscr U$ is non-free, then Definition \ref{MainDefinition} is trivial:
it prescribes that we compare intelligence of DLHAs by comparing their
performance in one fixed deterministic environment (because the corresponding social welfare
function is a dictatorship). We are mainly interested in the case where $\mathscr U$
is free\footnote{In particular, the strategy of this paper would break
down if the universe of deterministic environments were restricted to a finite set.}.

Note that Definition \ref{MainDefinition} already compares all possible DLHAs.
Thus we avoid a common intelligence comparison problem, wherein one compares some
limited set of agents, but then the addition of another agent changes the
intelligence order of the original agents.

An important difference between our approach and that of \citet{hutter2007}
is that when comparing the relative intelligence of two agents, we only concern ourselves
with the sets of environments where each agent outperforms the other, and we do not
concern ourselves with the numerical difference of the two agents' performance in any
environment. Thus, to us, if DLHA $A$ gains $10000$ more reward than DLHA $B$ on some
deterministic environment, that only helps $A$ the same amount as
if $A$ gained $0.00001$ more reward
than $B$ on that environment.
This is similar to how Elo ratings of chass-players
are based only on final outcomes (win, loss, draw),
ignoring any notion of the magnitude of a win or loss.
We would defend this choice by pointing out that for any
particular deterministic environment $e$, there is an
equivalent deterministic environment $e'$ which is identical to
$e$ in every way except that its rewards are all multiplied by $10000$, and there is also
an equivalent deterministic environment $e''$ which is
identical to $e$ in every way except that its
rewards are all multiplied by $0.00001$.

Note that we could have stated Definition \ref{MainDefinition} without the
precise structure of Definition \ref{EnvironmentAndAgent}.
All that is needed to state Definition \ref{MainDefinition} is that agents
outperform each other on environments, regardless of what exactly that actually means.
However, the precise structure of DLHAs and deterministic environments becomes
important when one
wants to actually obtain nontrivial \emph{structural properties} about the resulting
comparators (Section \ref{structuralpropertiessection}).

\begin{lemma}
\label{firsteasylemma}
    For every electorate $E=(\vec{e},\mathscr U)$ and DLHAs $A$ and $B$, exactly one of the following is true:
    \begin{enumerate}
        \item $A>_{E}B$.
        \item $A<_{E}B$.
        \item $A=_{E}B$.
    \end{enumerate}
\end{lemma}

\begin{proof}
    Let
    \begin{eqnarray*}
        X_1 &=& \{n\,:\,\mbox{$A$ earns more reward than $B$ on $e_n$}\},\\
        X_2 &=& \{n\,:\,\mbox{$A$ earns less reward than $B$ on $e_n$}\},\\
        X_3 &=& \{n\,:\,\mbox{$A$ earns the same reward as $B$ on $e_n$}\}.
    \end{eqnarray*}
    By Maximality (from Definition \ref{ultrafilterdefn}),
    either $X_1\in \mathscr U$ or $X_1^c\in\mathscr U$.

    Case 1: $X_1\in\mathscr U$. Then $X_2\not\in\mathscr U$ lest $X_1\cap X_2=\emptyset$ be in
    $\mathscr U$ (by $\cap$-closure), which would violate Properness of $\mathscr U$.
    Similarly, $X_3\not\in\mathscr U$. Altogether, $A>_{E}B$,
    $A\not<_{E}B$, and $A\neq_{E}B$.

    Case 2: $X_1\not\in\mathscr U$.
    By Maximality, either $X_2\in\mathscr U$ or $X_2^c\in\mathscr U$.

    Subcase 1: $X_2\in\mathscr U$. Then $X_3\not\in\mathscr U$, lest $X_2\cap X_3=\emptyset$
    be in $\mathscr U$, violating Properness. Altogether, $A<_{E}B$,
    $A\not>_{E}B$, and $A\neq_{E}B$.

    Subcase 2: $X_2^c\in\mathscr U$. By Maximality, either $X_3\in\mathscr U$ or
    $X_3^c\in\mathscr U$. We cannot have $X_3^c\in\mathscr U$, or else
    $X_1^c\cap X_2^c\cap X_3^c=\emptyset$
    would be in $\mathscr U$, violating Properness. So $X_3\in\mathscr U$.
    Altogether, $A=_{E}B$, $A\not>_{E}B$, and $A\not<_{E}B$.
\end{proof}

\begin{lemma}
    Let $E=(\vec{e},\mathscr U)$ be an electorate and let $A$ and $B$ be DLHAs.
    \begin{enumerate}
        \item The following are equivalent: $A=_{E}B$, $B=_{E}A$.
        \item The following are equivalent: $A>_{E}B$, $B<_{E}A$.
    \end{enumerate}
\end{lemma}

\begin{proof}
Straightforward.
\end{proof}

\begin{lemma}
    For every electorate $E=(\vec{e},\mathscr U)$ and DLHA $A$,
    $A=_{E}A$, $A\not>_{E}A$, and $A\not<_{E}A$.
\end{lemma}

\begin{proof}
    Straightforward (use the facts that $\emptyset\not\in\mathscr U$ by Properness,
    and $\mathbb N\in\mathscr U$ by Properness plus Maximality).
\end{proof}

When one defines a numerical intelligence measure for a single agent (as Legg and Hutter),
the corresponding comparison of two agents is automatically transitive.
If one instead chooses to directly compare two agents rather than measure a single agent,
one runs the risk of losing transitivity. The following lemma shows that we have
avoided that danger.

\begin{lemma}
\label{transitivityLemma}
    (Transitivity)
    Let $E=(\vec{e},\mathscr U)$ be an electorate and let $A,B,C$ be DLHAs.
    \begin{enumerate}
        \item If $A>_{E}B$ and $B>_{E}C$, then $A>_{E}C$.
        \item If $A<_{E}B$ and $B<_{E}C$, then $A<_{E}C$.
        \item If $A=_{E}B$ and $B=_{E}C$, then $A=_{E}C$.
    \end{enumerate}
\end{lemma}

\begin{proof}
    We prove the $>_{E}$ claim, the others are similar.
    Let
    \begin{eqnarray*}
        X_{AB} &=& \{n\,:\,\mbox{$A$ earns more reward than $B$ on $e_n$}\},\\
        X_{BC} &=& \{n\,:\,\mbox{$B$ earns more reward than $C$ on $e_n$}\},\\
        X_{AC} &=& \{n\,:\,\mbox{$A$ earns more reward than $C$ on $e_n$}\}.
    \end{eqnarray*}
    Since $A>_{E}B$, $X_{AB}\in\mathscr U$.
    Since $B>_{E}C$, $X_{BC}\in\mathscr U$.
    Now, for any $n$, if $A$ earns more reward than $B$ on $e_n$,
    and $B$ earns more reward than $C$ on $e_n$,
    then $A$ earns more reward than $C$ on $e_n$.
    This shows $X_{AC}\supseteq X_{AB}\cap X_{BC}$.
    By $\cap$-closure, $X_{AB}\cap X_{BC}\in\mathscr U$,
    so by Upward Closure, $X_{AC}\in\mathscr U$, that is, $A>_{E}C$.
\end{proof}

    \subsection{Comparing specialized intelligence}

    Definition \ref{MainDefinition} is flexible in that it
    allows one to choose which deterministic environments to consider.
    By considering specialized environments, one could obtain
    comparisons of specialized intelligence. For example,
    suppose $P_0,P_1,\ldots$ is a list of
    chess-playing computer programs. For each $i$, let $e_{2i}$
    (resp.\ $e_{2i+1}$)
    be the environment: ``play the white (resp.\ black) pieces
    in a game of chess against $P_i$, receiving $+1$ reward
    on the move (if any) in which you win, $-1$ reward
    on the move (if any) in which you lose, and $0$
    reward on all other moves'' (we can extend chess into an
    infinite game by obliging the players to play
    dummy moves forever after usual gameplay ends).
    For each ultrafilter, Definition \ref{MainDefinition} yields
    an abstract comparison of performance at the game of
    chess.

    A couple nuances about this paper will be illuminated by contrasting
    the above-mentioned comparison with practical chess-rating systems
    such as Elo ratings. Rating systems such as Elo
    cast chess-players as agents who can improve (or get worse)
    from game to game. By contrast, the DLHAs of
    Definition \ref{EnvironmentAndAgent} are fixed, and do not
    change from one environment to the next\footnote{This is reminiscent of
    an observation of \citet{good}: in an ideal chess-rating system,
    assuming players were accurately rated, the expected change to the ratings
    after any particular chess-gameplay ought to be $0$.}. This illustrates that
    Definition \ref{MainDefinition} is not really appropriate for
    the intelligence of human beings; it is more appropriate
    for the intelligence of a general-purpose utility-maximizing AI intended
    to be used in many different well-defined, rational, emotionless tasks.

    A primary reason we care about Definition \ref{MainDefinition} is
    because it is simple enough to allow the discovery
    of nontrivial, non-contrived structural properties.
    Something like Elo rating is infinitely more practical, but
    we doubt that it is possible to prove any nontrivial,
    non-contrived structural properties about performance as
    measured by Elo. Indeed, it is partly because
    our abstract definition treats agents as fixed and unchanging and deterministic
    that structural properties become obtainable. A chess-player's
    Elo rating at any given time depends not only on the outcomes
    that player has achieved, but also on the order of those outcomes,
    and on the orders of the outcomes achieved by the opponents
    at each play.

\section{Structural properties}
\label{structuralpropertiessection}

In this section, we will exhibit some nontrivial structural properties of
our intelligence comparators.
The properties we have been able to come up with are humble and few,
but we believe that in the context of Legg-Hutter-style intelligence measurement, they are
first of their kind.

\subsection{Properties of Teams}

The following definition is a special case of a more general definition which will
follow shortly.

\begin{definition}
\label{flexiblespecialcase}
For any DLHAs $A$ and $B$, we define a new
DLHA $A\oplus B$ such that for every observation-reward sequence
$(r_1,o_1,\ldots,r_n,o_n)$,
\[
(A\oplus B)(r_1,o_1,\ldots,r_n,o_n) =
    \begin{cases}
        A(r_1,o_1,\ldots,r_n,o_n) & \mbox{if $o_1$ is even,}\\
        B(r_1,o_1,\ldots,r_n,o_n) & \mbox{if $o_1$ is odd.}
    \end{cases}
\]
\end{definition}

One can think of $A\oplus B$ as a DLHA who plans to act as $A$ or $B$,
but has not yet decided which one. This DLHA will wait until seeing the first
observation in the environment before committing to act as $A$ or committing
to act as $B$. In an intuitive sense, $A\oplus B$ is a type of ``team'' formed
by $A$ and $B$.

\begin{proposition}
\label{EvenOddComboProposition}
Let $E=(\vec{e},\mathscr U)$ be an electorate
and let $A,A',B,B'$ be DLHAs.
If $A>_{E}A'$ and $B>_{E}B'$,
then $A\oplus B >_{E} A'\oplus B'$.
\end{proposition}

\begin{proof}
Let $V$ be the set of $n$ such that $e_n$ has first observation even, and let
$D$ be the set of $n$ such that $e_n$ has first observation odd.
By Maximality, either $V\in\mathscr U$, or $V^c=D\in\mathscr U$. We will assume
$V\in\mathscr U$, the other case is similar.

Let
\begin{eqnarray*}
    X_1 &=& \{n\,:\,\mbox{$A$ outperforms $A'$ on $e_n$}\},\\
    X_2 &=& \{n\,:\,\mbox{$A\oplus B$ outperforms $A'\oplus B'$ on $e_n$}\}.
\end{eqnarray*}
Since $A>_{E}A'$, $X_1\in\mathscr U$. By $\cap$-closure, $X_1\cap V\in\mathscr U$.
Now, for every $n\in X_1\cap V$, we have the following facts:
\begin{enumerate}
    \item
    $A$ outperforms $A'$ on $e_n$ (since $n\in X_1$).
    \item
    $e_n$'s first observation is even (since $n\in V$).
    \item
    $A\oplus B$ acts exactly like $A$ on $e_n$ (since $e_n$'s first observation is even).
    \item
    $A'\oplus B'$ acts exactly like $A'$ on $e_n$ (since $e_n$'s first observation is even).
\end{enumerate}
So for every $n\in X_1\cap V$, $A\oplus B$ outperforms $A'\oplus B'$ on $e_n$.
This shows $X_2\supseteq X_1\cap V$.
By Upward Closure, $X_2\in\mathscr U$, so $A\oplus B >_{E} A'\oplus B'$.
\end{proof}

Proposition \ref{EvenOddComboProposition} depends on the fact that we
compare intelligence only by where an agent outperforms,
without regard for the magnitude difference in rewards.
Otherwise, we could imagine a DLHA $A$ who performs slightly worse than $A'$
on even-numbered environments but makes up for it by clobbering $A'$ on
odd-numbered environments, making $A$ more intelligent than $A'$.
And we could imagine a DLHA $B$ who performs slightly worse than $B'$ on
odd-numbered environments but clobbers $B'$ on even-numbered environments.
Then $A\oplus B$ would perform worse than $A'\oplus B'$ everywhere, despite
each team-member being more intelligent than its counterpart.

The following definition generalizes Definition \ref{flexiblespecialcase}.

\begin{definition}
\label{flexiblecombo}
    Let $X$ be any set of reward-observation sequences and let $A$ and $B$ be DLHAs.
    The \emph{team combination of $A$ and $B$ given by $X$} is the DLHA
    $A\oplus_X B$ defined as follows.
    Suppose $r_1,o_1,\ldots,r_n,o_n$ is any observation-reward sequence.
    If for all $m\leq n$, $A(r_1,o_1,\ldots,r_m,o_m)=B(r_1,o_1,\ldots,r_m,o_m)$,
    then we declare $(A\oplus_X B)(r_1,o_1,\ldots,r_n,o_n)=A(r_1,o_1,\ldots,r_n,o_n)$.
    Otherwise, let $m\leq n$ be minimal such that
    $A(r_1,o_1,\ldots,r_m,o_m)\neq B(r_1,o_1,\ldots,r_m,o_m)$.
    We declare
    \[
        (A\oplus_X B)(r_1,o_1,\ldots,r_n,o_n) =
        \begin{cases}
            A(r_1,o_1,\ldots,r_n,o_n) & \mbox{if $(r_1,o_1,\ldots,r_m,o_m)\in X$},\\
            B(r_1,o_1,\ldots,r_n,o_n) & \mbox{otherwise}.
        \end{cases}
    \]
\end{definition}

In other words, $A\oplus_X B$ is the DLHA which has decided either to act as $A$,
or to act as $B$, but refuses to commit to one or the other until it is forced to.
In any particular deterministic environment, as long as the observations and
rewards are such that
$A$ and $B$ would act identically, then $A\oplus_X B$ acts in that way, without
committing to either one. Only when the observations and rewards are such that $A$
and $B$ would choose different actions, does $A\oplus_X B$ finally decide which DLHA
to follow in that environment, and it makes that decision based on whether or not
$X$ contains the observation-reward sequence which caused $A$ and $B$ to disagree.
Again, $A\oplus_X B$ is intuitively a type of ``team'' formed by $A$ and $B$.

The reader can easily check that if $X$ is the set of reward-observation sequences
with first
observation even, then $\oplus_X$, as given by Definition \ref{flexiblecombo},
is the same as $\oplus$ from Definition \ref{flexiblespecialcase}.

\begin{proposition}
\label{splittingproposition}
    For every electorate $E=(\vec{e},\mathscr U)$, DLHAs $A$ and $B$, and set $X$ of
    observation-reward sequences,
    $A\oplus_X B =_{E} A$ or $A\oplus_X B =_{E}B$.
\end{proposition}

\begin{proof}
    Let
    \[
        S = \{n\,:\,\mbox{$A\oplus_X B$ acts identically to $A$ on $e_n$}\}.
    \]
    By Maximality, either $S\in\mathscr U$ or $S^c\in\mathscr U$.

    Case 1: $S\in\mathscr U$.
    By Upward Closure,
    \[
        \{n\,:\,\mbox{$A\oplus_X B$ gets the same reward as $A$ on $e_n$}\}
        \supseteq
        S
    \]
    is also in $\mathscr U$, so $A\oplus_X B =_{E} A$.

    Case 2: $S^c\in\mathscr U$.
    By construction, for every $n\in S^c$, $A\oplus_X B$ acts identically to $B$ on $e_n$.
    So by similar reasoning as in Case 1, $A\oplus_X B =_{E} B$.
\end{proof}

Recalling that our intelligence comparators were ultimately motivated by
social welfare functions,
Proposition \ref{splittingproposition} might best be understood through
the aphorism: ``Looking like Lincoln to Lincoln-voters is
enough to get you elected, regardless how you look to Douglas-voters.''
$A\oplus_X B$ looks exactly like $A$ to one set of voters, and exactly like
$B$ to the opposite set of voters.

At first glance, Proposition \ref{splittingproposition} seems incompatible with
proposals based on weighing environments by Kolmogorov complexity. Surprisingly,
the incompatibility is not as big as initially appears.
Kolmogorov complexity depends on a reference universal Turing machine, and one could
contrive a reference universal Turing machine that gives unfair Kolmogorov complexity
to (say) environments with odd-numbered initial observations. This would cause
such environments to be under-represented in Kolmogorov-complexity-based intelligence,
so that the intelligence of $A\oplus B$ would be approximately the same as that of
$A$.

\begin{proposition}
\label{GeneralizedComboProposition}
    (Compare Proposition \ref{EvenOddComboProposition})
    Let $E=(\vec{e},\mathscr U)$ be an electorate,
    let $X$ be an observation-reward sequence set,
    and let $A,A',B,B'$ be DLHAs.
    Assume the following:
    \begin{eqnarray*}
        A &{>_{E}}& A',\\
        A &{>_{E}}& B',\\
        B &{>_{E}}& A',\\
        B &{>_{E}}& B'.
    \end{eqnarray*}
    Then $A\oplus_X B>_{E} A'\oplus_X B'$.
\end{proposition}

\begin{proof}
    By Proposition \ref{splittingproposition},
    $A\oplus_X B =_{E}A$ or $A\oplus_X B =_{E}B$,
    and $A'\oplus_X B' =_{E}A'$ or $A'\oplus_X B' =_{E}B'$.
    In any of the four cases, $A\oplus_X B >_{E} A'\oplus_X B'$
    by one of the four corresponding hypotheses.
\end{proof}

Comparing Propositions \ref{EvenOddComboProposition} and \ref{GeneralizedComboProposition},
we see that as the team notion becomes more general, the necessary hypotheses
seem to become more demanding.

One can imagine many other ways of forming teams; the teamwork notions we have considered
here are narrow. We have not managed to obtain structural properties for other
notions of teamwork. This might reflect the non-monolithic nature of real-world
intelligence, of which our intelligence comparators are a rather monolithic approximation.

\subsection{Properties of Quitters}
\label{quittersSubSection}

In this section, we will prove a couple of structural properties about
DLHAs who quit playing once they achieve a certain total reward. In order to
arrive at these results, some preliminary definitions are needed.
First, we will formalize a notion of a DLHA skipping a turn and what
it means for a deterministic environment to respect that.

\begin{definition}
    (Action-Skipping)
    \begin{enumerate}
        \item
            By ``skip'', we mean the the natural number $0$ (considered as an action).
        \item
            A deterministic environment $e$ is said to \emph{respect skipping} if
            for every action-sequence $a_1,\ldots,a_n$, if $a_n=\mbox{``skip''}$
            and $e(a_1,\ldots,a_n)=(r,o)$, then $r=0$
            (informally: $e$ always gives $0$ reward in response to a ``skip''
            action).
        \item
            An electorate $E=(\vec{e},\mathscr U)$ is said to \emph{respect skipping}
            if each $e_i$ respects skipping.
    \end{enumerate}
\end{definition}

The structural properties we are aiming for in this section also require a notion
of deterministic environments having a limit on how big of a reward they can give at any one
time (for simplicity we will make that limit be $1$, although this is
not particularly important).

\begin{definition}
    (Bounded rewards)
    \begin{enumerate}
        \item
            A deterministic environment $e$ is said to have \emph{bounded rewards}
            if for every action-sequence $a_1,\ldots,a_n$,
            if $e(a_1,\ldots,a_n)=(r,o)$, then $r\leq 1$.
        \item
            An electorate $E=(\vec{e},\mathscr U)$ is said to have
            \emph{bounded rewards} if each $e_i$ has bounded rewards.
    \end{enumerate}
\end{definition}

\begin{definition}
Suppose $A$ is a DLHA and $r$ is a real number.
We define a new DLHA $A\mathord{\restriction}_r$ as follows.
For every reward-observation sequence $r_1,o_1,\ldots,r_n,o_n$,
\[
A\mathord{\restriction}_r(r_1,o_1,\ldots,r_n,o_n)
=
\begin{cases}
    A(r_1,o_1,\ldots,r_n,o_n) & \mbox{if $r_1+\cdots+r_n<r$,}\\
    \mbox{``skip''} & \mbox{if $r_1+\cdots+r_n\geq r$.}
\end{cases}
\]
\end{definition}

We can think of $A\mathord{\restriction}_r$ as a version
of $A$ that becomes satisfied as soon as it
has achieved a total reward of at least $r$, after which
point $A\mathord{\restriction}_r$ takes it easy, ignores the environment,
and performs nothing but ``skip'' forever after.

\begin{proposition}
\label{quitterproposition}
Let $E=(\vec{e},\mathscr U)$ be an electorate with bounded rewards, and assume
$E$ respects skipping.
Suppose $A$ and $B$ are DLHAs with $A>_{E}B$. Let $r\in\mathbb R$.
If $B\mathord{\restriction}_r=_{E}B$, then $A\mathord{\restriction}_{r+1}>_{E}B$.
\end{proposition}

\begin{proof}
    Let
    \[
        X_1=\{n\,:\,\mbox{$B\mathord{\restriction}_r$ and $B$ get the same reward on $e_n$}\}.
    \]
    Since $B\mathord{\restriction}_r=_{E}B$, $X_1\in \mathscr U$.
    Let
    \[
        X_2=\{n\,:\,\mbox{$A$ outperforms $B$ on $e_n$}\}.
    \]
    Since $A>_{E}B$, $X_2\in\mathscr U$.
    By $\cap$-Closure, $X_1\cap X_2\in\mathscr U$.
    If I can show that $A\mathord{\restriction}_{r+1}$ outperforms
    $B$ on $e_n$ whenever $n\in X_1\cap X_2$,
    the proposition will be proved.

    Let $n\in X_1\cap X_2$.
    Since $n\in X_1$, $B$ and $B\mathord{\restriction}_r$ get the same reward on $e_n$.
    By construction, $B\mathord{\restriction}_r$ must get less than $r+1$ total reward on $e_n$
    (because $E$ has bounded rewards,
    so $e_n$ never gives a reward larger than $1$,
    and as soon as $B\mathord{\restriction}_r$ gets $r$ or more total reward,
    it begins playing ``skip'' forever after, which causes $e_n$ to give
    reward $0$ forever after, since $E$ respects skipping).
    Therefore, $B$ must get less than $r+1$
    total reward on $e_n$.
    Since $n\in X_2$, we know that $A$ outperforms $B$ on $e_n$. There are two cases.

    Case 1: $A$ does not get $r+1$ or more reward on $e_n$. Then by construction
    $A\mathord{\restriction}_{r+1}$ acts exactly like $A$ on $e_n$, so
    outperforms $B$ on $e_n$, as desired.

    Case 2: $A$ does get $r+1$ or more reward on $e_n$. It follows
    that $A\mathord{\restriction}_{r+1}$
    gets $r+1$ or more reward on $e_n$. Since we already
    established that $B\mathord{\restriction}_r$
    must get less than $r+1$ reward on $e_n$, this shows that $A\mathord{\restriction}_{r+1}$
    outperforms $B\mathord{\restriction}_r$ on $e_n$, and hence outperforms $B$ on $e_n$ since
    $B$ and $B\mathord{\restriction}_r$ get the same total reward on $e_n$.
\end{proof}

Intuitively, the way to think of Proposition \ref{quitterproposition} is as follows.
The fact that $B=_{E}B\mathord{\restriction}_r$ implies that
in ``almost every environment'' (or ``in an election-winning bloc of environments''),
$B$'s total reward is less than $r+1$.
So if $A$'s only objective is to beat $B$, $A$ might as well relax and take it easy
any time $A$ has already achieved at least $r+1$ total reward: by that point, $A$ has
already beaten $B$ on the environment in
question (with ``almost no environments'' being exceptions).

\begin{definition}
    A deterministic environment $e$ is \emph{merciful} if
    the rewards it outputs are never negative.
    An electorate $E=(\vec{e},\mathscr U)$ is \emph{merciful}
    if each $e_i$ is merciful.
\end{definition}

\begin{proposition}
\label{trivialquitterproposition}
    Let $E=(\vec{e},\mathscr U)$ be a merciful electorate that respects skipping.
    For every DLHA $A$ and real number $r$, $A\not<_{E}A\mathord{\restriction}_r$.
\end{proposition}

\begin{proof}
    Since $E$ is merciful,
    it follows that $A$ cannot possibly get less total reward than $A\mathord{\restriction}_r$
    on any $e_n$, and the proposition trivially follows.
\end{proof}

The proof of Proposition \ref{trivialquitterproposition} would not go through if
the deterministic environments were not merciful.
This brings the following interesting fact to our attention:
if deterministic environments are allowed to give out punishments,
then the relative intelligence of two DLHAs really ought to depend on
context, namely, on how risk-averse we are.
After all, if $A$ is a DLHA, then who
is more intelligent: $A$ or $A\mathord{\restriction}_1$?
If it is better to safely achieve a total reward of at
least $1$, at the price of forgoing higher rewards with a higher risk,
then $A\mathord{\restriction}_1$
should be more intelligent, because in some deterministic environments,
$A$ could temporarily achieve rewards adding to $\geq 1$, but then later
receive large punishments. On the other hand, if
it is better to achieve highest-possible rewards, even at risk of being punished
in some environments, we might prefer $A$, because there may be many deterministic environments
where $A\mathord{\restriction}_1$ will fall well short of $A$.
Risk-aversion is, itself, already a potentially multi-dimensional parameter, and we
expect it is probably only one of many hidden parameters
underlying intelligence. The choice of an ultrafilter in Definition \ref{MainDefinition}
serves as a convenient catch-all parameter which, we think, might serve to wrap up all
these different complicated parameters into one.

\section{Discussion}
\label{objectionsection}

In this section, we will discuss some objections to our comparators and to
Legg and Hutter's agent-environment model. We thank an anonymous reviewer for
these objections.

\subsection{What does this definition really have to do with intelligence?}

Out of context, Definition \ref{MainDefinition} might not appear to have
anything to do with intelligence. There are two steps to see its relationship
to intelligence. First, that the agent-environment model
(Definition \ref{EnvironmentAndAgent}) is closely related to a type of
intelligence; second, that Definition \ref{MainDefinition} is closely related
to the problem of using Definition \ref{EnvironmentAndAgent} to compare intelligence.

\citet{hutter2007} have done a fine job arguing for the
close relationship between the agent-environment model and a type of intelligence.
We will add the following.
Whatever intelligence is, it surely involves components such as creativity,
pattern recognition, the ability to abstract and to extrapolate, etc.
Whether or not a DLHA is capable of creativity and so on is beyond the scope of
this paper, but it seems clear that there are deterministic environments
\emph{intended}\footnote{We emphasize the intention part
of this argument because to assume that an environment can \emph{succeed} in
rewarding creativity would be begging the question.} to reward creativity, or
to reward pattern
recognition, or to reward other specific components of intelligence.
To consistently extract good rewards from many such environments, a DLHA would
either need to possess
those components of intelligence, or at least appear to possess them for
practical intents and purposes. To quote Socrates:
``But, not possessing right judgment, you would not realize that you are
enjoying yourself even while you do, and, being unable to calculate,
you could not figure out any future pleasures for yourself''
\citep{philebus}.

It remains to show how Definition \ref{MainDefinition} is closely related to
the problem of using Definition \ref{EnvironmentAndAgent} to compare intelligence.
This is a matter of summarizing the argument from Section \ref{maindefinitionsection}.
\begin{itemize}
    \item
    We would like to use the relative performance of DLHAs in deterministic environments
    to determine which DLHAs are more intelligent. More intelligent DLHAs should
    win higher rewards, not in every deterministic environment, but in some aggregate
    sense across a whole universe of
    deterministic environments. This idea sounds great, but
    it is unclear how to realize it.
    \item
    We view the problem from a different angle: rather
    than saying a DLHA with higher intelligence achieves higher rewards
    on average from deterministic environments, instead say that deterministic
    environments \emph{elect} more intelligent machines. For, what else can it
    possibly be when a winner is chosen by the aggregated judgments of many
    judges? In this way, we obtain a bridge linking AGI to election theory, a literature
    from which we can import centuries of research.
    The question ``Who is more intelligent?'' is transformed into the
    question ``Who wins the election?''
\end{itemize}

Kirman and Sondermann exhaustively characterized ways of determining election
winners, under additional constraints. It is debateable whether those
additional constraints are appropriate, but if we assume they are,
then Definition \ref{MainDefinition} is immediate. Whether or not the constraints
are appropriate, Kirman and Sondermann can be used as a heuristic in order to
obtain a definition. If the constraints are appropriate, our definition might be
the only possible definition, subject to those constraints!
If the constraints are not appropriate, then our definition is still better than
nothing, even if someday a better definition is found (one which would necessarily
have to violate those constraints).


\subsection{Why give a new definition if Legg and Hutter already gave one?}

The intelligence measure defined by Legg and Hutter is important because it showed
that it is in principle possible to formalize an abstract measurement of intelligence.
And it is more computationally tractable than our comparators.
There are at least two problems with the specific measure which Legg and Hutter proposed.
\begin{enumerate}
\item
Legg and Hutter's measure depends on a
reference universal Turing machine \citep{leike2015bad}.
\item
Legg and Hutter's measure is biased toward agents who perform well in simple environments.
For example, almost all environments have Kolmogorov complexity $>100$,
but the Legg-Hutter intelligence of an agent is almost completely determined by its
performance on environments with Kolmogorov complexity $\leq 100$.
\end{enumerate}
We are skeptical that any nontrivial, non-contrived
\emph{structural properties} will be proven about Legg and Hutter's intelligence measure.
Any such property would either have to fundamentally hinge on the specific
choices that went into their definition (which we think unlikely),
or else be so general as to work in spite of those specific choices
(any such property would surely be weak).

\subsection{What about emotions, emotional intelligence, irrational intelligence, etc?}

This paper is concerned with pragmatic intelligence, appropriate for the situation where one
is evaluating a general-purpose utility-maximizing AI intended to be used in many different
well-defined, rational, emotionless tasks. Our definition does not account for
emotions (if any) in the machine, nor for what the machine knows or believes
or understands (if anything), except indirectly to the extent that such emotions,
knowledge, beliefs, or understanding help the machine to better extract rewards
from reward-giving environments.
It could potentially be dangerous to apply these definitions to human beings.

\subsection{Not all instances of intelligence are ordinarily comparable}

For any particular electorate, Definition \ref{MainDefinition} makes all
DLHAs comparable, but we should stress that if $E$ and $E'$ are distinct
electorates, then $A>_EB$ does not generally have any bearing on whether
or not $A>_{E'}B$. Thus, the fact that not all instances of intelligence
are ordinarily comparable is actually evidence in \emph{favor} of the
parametrized-family approach of this paper. Different electorates might
measure different types of intelligence, and a particular DLHA might
be more intelligent according to one electorate and less intelligent
according to another.

\subsection{What about scholars who do not consider intelligence amenable to definition?}

We understand that some scholars consider intelligence immune to
definition, end-of-story. At first glance this paper
might seem opposed against those
scholars. But in actuality, this paper \emph{reconciles} those
scholars with scholars who think intelligence can be defined.
Except for degenerate edge-cases, the intelligence comparators of
Definition \ref{MainDefinition} depend on a \emph{free ultrafilter}.
Mathematical logicians have formally proven that, although free ultrafilters
exist, it is impossible to exhibit one. Thus in a real and
formal sense we vindicate the non-definition scholars, while at the same
time showing (by proving some structural properties) that
all is not lost.

\subsection{Concerns about infinity}

In the Legg-Hutter model, an agent interacts with an environment for an eternity.
Readers might be concerned that this is inaccurate, considering that our universe
will eventually end, etc. Likewise, we work over the whole infinite space of
all hypothetical environments, even though only finitely many environments can
exist in reality (if our universe has only finitely many atoms and those
atoms can only be arranged in finitely many ways). We acknowledge these shortcomings,
but we defend them by comparing the situation to that of Turing machines. Turing
machines, almost universally accepted as a model of computation, are
equipped with infinite memory-tapes and allowed to run for eternity. And
computability theorists routinely consider the entire infinite space of hypothetical
Turing machines, even though a finite, discrete universe can only
instantiate finitely many of them.
Infinity serves as an important simplifying assumption. Without such simplifying
assumptions, computer science would be filled with arbitrary decisions about
how many bits of memory the idealized computer actually has, how long
algorithms can run before the heat-death of the universe, etc. Considering
how little we really know about these questions, and how rapidly our opinions about
them change, we opine that the resulting science would be weakened.

For additional remarks about the justification and usefulness of infinity-based
simplifying assumptions in (biological) science, see section 2 of
\cite{alexander2013}.

\subsection{Drawing conclusions from environments that don't actually exist}

As a followup to the previous subsection about infinity, the reader might
question how relevant it is to consider environments that cannot possibly occur
(due to, e.g., the finiteness of our universe) when comparing intelligence.
We would like to propose a thought experiment. Suppose a passenger
trapped in the sinking Titanic is a world-champion chess-player. Arguably,
in the moments before the Titanic sinks, it is virtually impossible that this
passenger will ever play another game of chess. Should that mean
that their chess abilities are irrelevant to their intelligence? The answer could
go either way, and it might indeed be interesting to study notions of intelligence
in which the model takes into account what sort of scenarios the agent can
possibly find itself in. In this paper, we take the tack that being a world-champion
chess-player \emph{is} relevant to one's intelligence, even when one is in a
predicament that makes any future chess games impossible. We should stress that
this thought experiment only serves as an analogy, and this paper is not really
appropriate for comparing the intelligence of human beings, but rather the
intelligence of DLHAs.

\subsection{Shouldn't the model consider future discounting?}

Readers might object that rewards early-on in a DLHA's encounter with a deterministic
environment should carry more weight than rewards later-on.

In our opinion, time discounting is a complicated proposition, packed with
arbitrary decisions of an infinite-dimensional nature, none of which (according to what
we know so far) have any intrinsic relationship to intelligence.
Just as, in physics, an equation can be simplified by replacing a product of
many physical constants by one single new constant, in the same way,
the ultrafilter in Definition \ref{MainDefinition} serves as an elegant master-parameter
which hides within it many complicated choices and parameters, including those
related to future discounting. For example, an ultrafilter might be chosen which
gives preference to environments whose rewards taper off exponentially.
As a general principle, any time a thing must unavoidably depend on arbitrary
choices, the philosopher should define that thing with those choices parametrized away,
rather than embedding them within the definition itself.

\subsection{Could this definition be used to compare intelligence of weather systems,
stars, etc.?}

Definition \ref{MainDefinition} is a formal definition which applies to
DLHAs (i.e., to functions with specific domain and range), not to things like
weather systems or stars. Possibly, a DLHA could be contrived which takes
actions somehow derived from the behavior of a formalized weather system or star.
We conjecture that according to most non-contrived electorates, such a DLHA
would compare unfavorably against any decent AGI.


\section{Open Questions about other intelligence proposals}
\label{openquestionsection}

Even if the reader completely disagrees that our proposal has anything to do with
intelligence, nevertheless, the results we have obtained still serve as
a useful source of open questions which can be asked about other proposals for
measuring the intelligence of Legg-Hutter-style agents.
If nothing else, answers to these questions could ground a rich comparative
vocabulary for the discussion of intelligence measurement systems.

\subsection{Legg and Hutter universal intelligence}

For every agent $A$, let $\varGamma(A)$ denote the so-called universal intelligence of $A$
as defined by \citet{hutter2007}.

\begin{question}
\label{FirstOpenQuestion}
    (Compare Proposition \ref{EvenOddComboProposition})
    Let $A,A',B,B'$ be agents with $\varGamma(A)>\varGamma(A')$,
    $\varGamma(B)>\varGamma(B')$.
    Is it necessarily true that $\varGamma(A\oplus B)>\varGamma(A'\oplus B')$?
\end{question}

\begin{question}
    (Compare Proposition \ref{EvenOddComboProposition})
    Let $A,A',B,B'$ be agents
    with $\min(\varGamma(A),\varGamma(B))>\max(\varGamma(A'),\varGamma(B'))$.
    Is it necessarily true that $\varGamma(A\oplus B)>\varGamma(A'\oplus B')$?
\end{question}

\begin{question}
    (Compare Proposition \ref{GeneralizedComboProposition})
    Let $A,A',B,B'$ be agents
    with $\min(\varGamma(A),\varGamma(B))>\max(\varGamma(A'),\varGamma(B'))$,
    and let $X$ be a set of reward-observation sequences.
    Is it necessarily true that $\varGamma(A\oplus_X B)>\varGamma(A'\oplus_X B')$?
\end{question}

\begin{question}
\label{LastOpenQuestion}
    (Compare Proposition \ref{quitterproposition})
    Let $\varGamma'$ be the same as $\varGamma$ except that $\varGamma'$ only considers
    environments which respect skipping and have bounded rewards.
    Suppose $A$ and $B$ are agents with $\varGamma'(A)>\varGamma'(B)$.
    Let $r\in\mathbb R$
    be such that $\varGamma'(B\mathord{\restriction}_{r})=\varGamma'(B)$.
    Is it necessarily true that
    $\varGamma'(A\mathord{\restriction}_{r+1})>\varGamma'(B)$?
\end{question}

\subsection{Hern\'andez-Orallo and Dowe universal intelligence}

Various proposals for so-called universal intelligence are given by \citet{hernandez}.

\begin{question}
    What are the answers to Questions \ref{FirstOpenQuestion}--\ref{LastOpenQuestion}
    if $\varGamma$ denotes the various intelligence definitions
    given by \citet{hernandez}?
\end{question}

\subsection{Hibbard universal intelligence}

Hibbard has proposed \citep{hibbard} a definition of so-called universal intelligence
of quite a different flavor.

\begin{question}
\label{ReallyLastOpenquestion}
    What are the answers to Questions \ref{FirstOpenQuestion}--\ref{LastOpenQuestion}
    if $\varGamma$ denotes the intelligence measure given by
    \citet{hibbard}?
\end{question}

\section{Conclusion}
\label{conclusionsection}

Following \citet{hutter2007}, we considered an abstraction in
which agents take actions within environments, based on the observations and
rewards they receive from those environments.
Unlike Legg and Hutter, we constrained both agent and environment to be
deterministic.
Various authors
\citep{hutter2007, hernandez, hibbard} have proposed
ingenious means of quantifying the intelligence
of such agents.
Informally, the goal of such a measurement is that
an agent with higher universal intelligence should realize larger rewards
over the (infinite) space of environments, in some aggregate
sense.

We proposed a new approach to comparing the intelligence of
deterministic Legg-Hutter agents (DLHAs), a more
abstract approach which lends itself to proving certain structural properties
about intelligence. Rather than measure the universal intelligence of a
DLHA, we focused on a closely related problem: how to compare the relative
intelligence of two DLHAs. Our approach is based on the following
insight: deterministic environments can be treated as voters in an election to determine
which DLHA (if either) is more intelligent. Each environment votes for that
DLHA which scores the highest reward in that environment (or votes that the
two DLHAs are equally intelligent, if both DLHAs score equally in it).
This realization provides an exciting bridge connecting the budding field of
Legg-Hutter-style intelligence to
the mature field of election theory.
In particular,
\citet{kirman} have completely characterized
reasonable solutions to the infinite-voter election problem in terms of \emph{ultrafilters},
a device from mathematical logic. This led us to an elegant means
of using ultrafilters to compare intelligence of DLHAs.

The intelligence comparators we arrived at are elegant enough that
we were able to prove some structural properties about how DLHAs'
intelligences are related \emph{in general}. For example, if $A$ and $B$ are
DLHAs, more intelligent than DLHAs $A'$ and $B'$ respectively, we proved
(Proposition \ref{EvenOddComboProposition}) that
$A\oplus B$ is more intelligent than $A'\oplus B'$, where $\oplus$ is an
operator which takes two DLHAs and outputs a new DLHA which can roughly
be thought of as a ``team'' made up of $A$ and $B$
(Definition \ref{flexiblespecialcase}).
In short (although this is an oversimplification):
``If a team's members are more intelligent, then that team is more intelligent.''

Our comparators are purely theoretical and not useful
for doing concrete computations. Nevertheless, the structural
properties we are able to prove about these theoretical comparators
are a useful source of open questions about other Legg-Hutter-style
intelligence approaches (Questions \ref{FirstOpenQuestion}--\ref{ReallyLastOpenquestion}).
Our hope is that this will inspire more research on Legg-Hutter-style intelligence
even among readers who disagree about our particular proposal.

\section*{Acknowledgements}

We acknowledge Adam Bloomfield, Jordan Fisher,
Jos{\'e} Hern{\'a}ndez-Orallo,
Bill Hibbard, Marcus Hutter, Peter Sunehag, and the reviewers
for feedback and discussion.

\bibliography{biblio}

\end{document}